\documentclass{article}
\usepackage[utf8]{inputenc}

\usepackage{amsthm,amssymb,amsmath}
\usepackage{natbib}

\theoremstyle{plain}

\newtheorem{corollary}{Corollary}
\newtheorem{proposition}{Proposition}

\usepackage{hyperref}
\usepackage{graphicx}

\title{ Forecast collapse of transformer-based models 
 under squared loss in financial time series}

\begin{document}

\maketitle

\begin{center}

\vspace{1cm}

{\large
\textbf{Pierre Andreoletti}
}

\vspace{0.3cm}

{\it
Institut Denis Poisson, UMR C.N.R.S. 7013, France.
}

\vspace{0.2cm}

{\tt
\href{mailto:Pierre.Andreoletti@univ-orleans.fr}{Pierre.Andreoletti@math.cnrs.fr}
}

\end{center}

\vspace{1cm}

\noindent\textbf{Abstract.}
We study trajectory forecasting under squared loss for time series with weak conditional structure, using highly expressive prediction models. 
Building on the classical characterization of squared-loss risk minimization, we emphasize regimes in which the conditional expectation of future trajectories is effectively degenerate, leading to trivial Bayes-optimal predictors (flat for prices and zero for returns in standard financial settings).

In this regime, increased model expressivity does not improve predictive accuracy but instead introduces spurious trajectory fluctuations around the optimal predictor. These fluctuations arise from the reuse of noise and result in increased prediction variance without any reduction in bias. 

This provides a process-level explanation for the degradation of Transformer-based forecasts on financial time series.

We complement these theoretical results with numerical experiments on high-frequency EUR/USD exchange rate data, analyzing the distribution of trajectory-level forecasting errors. The results show that Transformer-based models yield larger errors than a simple linear benchmark on a large majority of forecasting windows, consistent with the variance-driven mechanism identified by the theory.


\vspace{0.5cm}



\vspace{0.5cm}

\section{Introduction}

Transformer-based architectures have recently become the dominant paradigm for long-term time-series forecasting.
Following the success of attention mechanisms in natural language processing, numerous works have proposed specialized
Transformer variants tailored to temporal data, aiming to address long-range dependencies, non-stationarity, and scalability.

Prominent examples include Informer \cite{zhou2021informer}, Autoformer \cite{wu2021autoformer}, FEDformer \cite{zhou2022fedformer}, and PatchTST \cite{nie2023patchtst}. These models report strong empirical performance on a variety of structured forecasting benchmarks, including electricity load, traffic flow, weather variables, and sensor measurements.

These models share a common supervised learning formulation: given a fixed-length window of past observations,
the model is trained to predict a multi-step future trajectory by minimizing a regression loss, typically the mean
squared error (MSE).
Architectural innovations including sparse attention, trend seasonal decomposition, frequency-domain modeling,
or patch-based tokenization are introduced to facilitate optimization and representation, while the learning
objective itself remains unchanged.

When the same Transformer-based models are applied to financial time series, however, a markedly different behavior
is consistently observed.
Despite their high expressive power, multi-step forecasts of asset prices tend to collapse toward trivial dynamics,
remaining close to the last observed value, while return forecasts concentrate near zero.

This phenomenon has been widely documented in empirical studies on financial forecasting with deep learning models,
including Transformer-based architectures.
Existing explanations are largely heuristic, invoking low signal-to-noise ratios, non-stationarity, regime changes,
or limited data availability.
As a result, much of the literature implicitly assumes that improved architectures, stronger regularization,
or larger datasets may eventually overcome these difficulties.

In this work, we argue that such expectations are fundamentally misplaced for a broad and practically relevant
class of learning objectives.
We show that the observed forecast collapse is not an artifact of optimization, architectural inadequacy, or data
scarcity.
Rather, it is the theoretically expected outcome of empirical risk minimization (ERM) applied to trajectory
forecasting under squared loss, when the underlying stochastic process exhibits unpredictable innovations a
standard modeling assumption in finance.

Our analysis is conducted at the process level, where both inputs and outputs are temporal segments.
Within this framework, we recall and formalize the fact that, under squared
trajectory loss, the Bayes-optimal predictor is given by the conditional mean of
the future trajectory given the observed past.
While this result is classical in itself, its implications for financial
time-series forecasting with Transformer-based models have not been explicitly
formalized.
Under mild assumptions such as conditional mean-zero returns or martingale price dynamics, we show that this
conditional mean trajectory is provably trivial: flat for prices and identically zero for returns.

Consequently, any sufficiently expressive Transformer-based model trained via ERM regardless of architectural
refinements such as patching, frequency decomposition, or sparse attention is expected to converge toward such
trivial forecasts as data volume increases and optimization improves.

The mechanism underlying this phenomenon is closely related to classical observations on the behavior of highly expressive models trained by empirical risk minimization.

In static regression settings, it is well known that overly flexible predictors may amplify noise and exhibit poor generalization performance. This effect is captured by classical bias--variance trade-offs \cite{Geman1992,Hastie2009}, and more explicitly by interpolation-based methods such as nearest neighbors, which reuse training noise in their predictions \cite{Stone1977,Devroye1996}. More recent work on interpolating estimators and double descent has further shown that zero training error does not necessarily harm generalization when sufficient structure is present in the data \cite{Belkin2019,Bartlett2020}.

However, these results are primarily established in non-temporal settings with scalar outputs, and do not directly address trajectory forecasting problems involving temporal dependence and multi-horizon targets.

In contrast, time-series forecasting involves temporal dependence, overlapping input windows, and multi-horizon targets. In particular, when the conditional structure of the future is weak, as is often the case in financial time series, the role of model expressivity differs fundamentally. In this regime, the noise-reuse mechanism induced by highly expressive models does not vanish asymptotically, but instead leads to systematic fluctuations around a trivial optimal predictor.

Despite the empirical success of Transformer-based architectures on structured benchmarks such as electricity load or traffic forecasting, there remains little theoretical understanding of their behavior in such weakly predictable regimes. Most existing financial applications remain empirical and report limited and unstable predictive gains when applying complex machine learning models \cite{Gu2020,DePrado2018}.

Our contribution addresses this gap by providing a formal analysis of empirical risk minimization over nested hypothesis classes, ranging from simple well-specified models to highly expressive interpolating predictors, in an explicitly temporal forecasting setting.

We show that when the Bayes-optimal predictor is simple and the future is dominated by irreducible noise, enlarging the hypothesis class to include interpolating predictors can strictly increase the expected prediction error, despite increased expressive power.
We then investigate this phenomenon empirically on high-frequency EUR/USD exchange rate data using a PatchTST model, analyzing the distribution of trajectory-level forecasting errors. The results show that the Transformer-based predictor yields almost systematically larger errors than a simple linear benchmark across most of the distribution, consistent with spurious fluctuations around the flat predictor predicted by the theory.

\subsection*{Notations and hypothesis}

We consider a discrete-time stochastic process $(X_t)_{t\in\mathbb{Z}}$
defined on a probability space $(\Omega,\mathcal G,\mathbb P)$
and adapted to a filtration $(\mathcal G_t)_{t\in\mathbb{Z}}$.

Fix two integers $L \ge 1$ (lookback window) and $H \ge 1$ (forecast horizon).
At time $t$, the learner observes the past trajectory
\[
\mathbf X_t^{(L)} := (X_{t-L+1},\dots,X_t) \in \mathbb R^L,
\qquad
\mathcal F_t := \sigma(\mathbf X_t^{(L)}) \subseteq \mathcal G_t,
\]
where $\mathcal F_t$ represents the information available to the learner at time $t$
(formally, the $\sigma$-algebra generated by $\mathbf X_t^{(L)}$),
and aims to predict the future trajectory
\[
\mathbf Y_t^{(H)} := (X_{t+1},\dots,X_{t+H}) \in \mathbb R^H.
\]

This framework encompasses a wide range of time-series forecasting applications,
including energy consumption, traffic flow, environmental measurements, and
financial prices, also this trajectory-to-trajectory formulation is standard in long-term forecasting with Transformer-based models, including Informer, Autoformer, and PatchTST. 

A predictor of the future trajectory is thus a measurable function
$f : \mathbb R^L \to \mathbb R^H$ mapping the observed past window
to a multi-step forecast.
Such predictors are typically modeled as elements of a hypothesis class,
often represented in practice by parametric families of functions,
for instance neural networks.

Formally, we model a neural network architecture as a parametric family
of measurable functions.
Let $\mathcal U \subset \mathbb R^L$ and $\mathcal Y \subset \mathbb R^H$
denote the input and output spaces, respectively, and let
$\Theta \subset \mathbb R^p$ be a parameter space (with $p$ a positive integer).
Each parameter $\theta \in \Theta$ defines a predictor
$f_\theta : \mathcal U \to \mathcal Y$,  the set of these functions  is called hypothesis class :
\[
\mathcal H
:=
\{ f_\theta : \mathcal U \to \mathcal Y \mid \theta \in \Theta \}.
\]

Thus, $\mathcal H$ is the set of all trajectory-to-trajectory predictors
that can be realized by the considered model architecture for some
parameter choice.
Our analysis depends only on this induced hypothesis class $\mathcal H$,
and not on the specific parametrization or training algorithm. \\
Given a predictor $f\in\mathcal H$, the model produces at time $t$ a forecast of
the future trajectory by evaluating $f$ on the observed past window
$\mathbf X_t^{(L)}$, yielding the prediction $f(\mathbf X_t^{(L)})\in\mathbb R^H$. Then, the learning objective we consider here is defined through a trajectory-level
squared loss $\ell$,
\[
\ell\big(\mathbf{Y}_t^{(H)},f(\mathbf{X}_t^{(L)})\big)
=
\|\mathbf{Y}_t^{(H)} - f(\mathbf{X}_t^{(L)})\|_2^2
:=
\sum_{h=1}^H \big(X_{t+h} - f_h(\mathbf{X}_t^{(L)})\big)^2,
\]
where $f_h$ is the coordinate $h$ of $f$. The corresponding population risk is the mean of the above loss
\[
\mathcal{R}(f)
:=
\mathbb{E}\Big[\|\mathbf{Y}_t^{(H)} - f(\mathbf{X}_t^{(L)})\|_2^2\Big].
\]

Considering data, we introduce the corresponding estimator of $\mathcal{R}$ : let $t_1,\dots,t_n$ denote observation times used for training.
The associated learning sample is denoted 
\[
\mathcal D_n := \bigl\{ (\mathbf X_{t_i}^{(L)}, \mathbf Y_{t_i}^{(H)}) \bigr\}_{i=1}^n,
\]
and the empirical risk of a predictor $f \in \mathcal{H}$ is defined as
\[
\hat{\mathcal{R}}_n(f)
:=
\frac{1}{n}\sum_{i=1}^n
\ell\bigl(\mathbf{Y}_{t_i}^{(H)},\, f(\mathbf{X}_{t_i}^{(L)})\bigr).
\]

To ensure that empirical risk minimization yields meaningful guarantees,
we assume that the empirical risk provides a uniform approximation of the
population risk over the hypothesis class.

\noindent Specifically, we assume that $\mathcal H$ is (weak) \emph{Glivenko-Cantelli with
respect to the loss $\ell$}, that is,
\begin{align}
\sup_{f\in\mathcal H}
\left|
\hat{\mathcal R}_n(f)-\mathcal R(f)
\right|
\rightarrow 0
\qquad \text{in probability as } n \to \infty. \label{ERMC}
\end{align}

 \noindent  This assumption is standard in statistical learning theory and ensures that
empirical risk minimization is asymptotically consistent within the class
$\mathcal H$.

Then, if the empirical risks converge uniformly to their population
counterparts over $\mathcal H$, then any empirical risk minimizer achieves
asymptotically the optimal population risk over $\mathcal H$. More precisely, let
\begin{align}
\hat f_n \in \arg\min_{f\in\mathcal H} \widehat{\mathcal R}_n(f)
\label{RiskM}
\end{align}
be a (measurable) empirical risk minimizer.
This estimator is data-dependent and targets the best-in-class predictor under the empirical trajectory risk. Then if $\mathcal H$ is Glivenko--Cantelli with respect to the loss $\ell$, then
\begin{align}
|\mathcal R(\hat f_n) - \inf_{f\in\mathcal H} \mathcal R(f)| \rightarrow 0 \quad \text{in\ probability}. \label{f_nconsist}
\end{align}
Note that for a fixed predictor $f$, the population risk $\mathcal R(f)$ is deterministic. However, since $\hat f_n$ depends on the training sample, the quantity $\mathcal R(\hat f_n)$ is itself random. The expectation in the definition of $\mathcal R$ is taken with respect to a fresh draw $(X,Y)$ from the data-generating process, independently of the training data $\mathcal D_n$. \\
This convergence ensures that, under uniform convergence, empirical risk minimization is an asymptotically consistent procedure for minimizing the population risk over the hypothesis class. For completeness, a proof of this fact is given in the Appendix A.1.

To identify the limit of empirical risk minimization, we now characterize
the population risk minimizer over the hypothesis class. Let $(U,V)$ be a generic input-output pair with the same distribution as
$(\mathbf X_t^{(L)},\mathbf Y_t^{(H)})$.
The population risk for $(U,V)$ can then be written as
\[
\mathcal R(f) = \mathbb E\!\left[\|V - f(U)\|_2^2\right].
\]
\noindent Under squared loss, a minimizer of the population risk $\mathcal R(f)$ over all measurable predictors is given by the conditional mean, for any $u \in \mathbb R^L$
\[
m(u) := \mathbb E[V \mid U=u].
\]
For completeness, a short proof is given in Section~\ref{sec2}.
We also assume that the model is \emph{well specified}, in the sense that the Bayes predictor belongs to the hypothesis class, i.e.,
\[
m \in \mathcal H.
\]
Denoting $f^\star := m$, it follows that
\[
\inf_{f\in\mathcal H}\mathcal R(f) = \mathcal R(f^\star).
\]
Hence, by~\eqref{f_nconsist},
\[
\left| \mathcal R(\hat f_n) - \mathcal R(f^\star) \right| \rightarrow 0
\qquad \text{in probability as } n \to \infty.
\]

\medskip
\medskip
\subsection*{Main results and organization of the paper.}
Our results can be summarized as follows.

In Section~\ref{sec2}, we first recall a classical fact: under squared trajectory loss, the population-risk minimizer is the conditional mean trajectory of the future given the observed past. Our contribution is not this fact itself, but the emphasis on its consequences for forecasting regimes with weak conditional structure. In particular, we show that this learning objective is well aligned with structured time series, where the future retains a non-trivial dependence on the past, but becomes effectively degenerate when the innovations are too weakly predictable from past information, as is typically the case in finance. Under standard martingale-type assumptions, this implies that the Bayes-optimal predictor is trivial: flat for prices and zero for returns.

In Section~3, we then study the statistical consequences of this phenomenon for model selection. We compare a simple well-specified linear predictor with a much richer class containing interpolating estimators. Our analysis shows that, in the weak-signal regime described above, the additional fluctuations produced by highly expressive predictors are not informative signal, but model-induced variability generated through the reuse of noise. In particular, this excess variability is larger than in the linear case and results in strictly worse out-of-sample prediction performance.

Finally, in Section~4, we confront these theoretical predictions with numerical experiments on high-frequency EUR/USD data, a setting whose statistical properties are representative of standard financial forecasting problems. We show that the empirical distribution of trajectory-level forecasting errors is consistent with the theory: the Transformer-based predictor exhibits systematically larger errors than a simple linear benchmark.

For readability, the proofs of the auxiliary consistency results and technical bounds are collected in the Appendix.

\section{Succeed and collapse forecasting for time series \label{sec2}}

\subsection*{When does trajectory forecasting succeed ?}

Recall that throughout this work, we consider the squared trajectory loss
 $\ell(y,\hat y)=\|y-\hat y\|_2^2$ with $y,\hat y\in\mathbb R^H.$
In our forecasting setting, $y=\mathbf Y_t^{(H)}$ represents the future trajectory
$(X_{t+1},$ $\dots,X_{t+H})$, while $\hat y=f(\mathbf X_t^{(L)})$ denotes its
prediction based on the observed past.
So this loss treats all future time steps symmetrically and penalizes pointwise
deviations over the entire forecast horizon.
As a consequence, minimizing the associated risk implicitly assumes that the
relevant predictive information is contained in the conditional mean of the
future trajectory given the observed past.

Such an objective is well suited to settings where the process admits a
structured and predictable component, such as seasonal or trend-driven dynamics
commonly encountered in electricity consumption, traffic flows, or other
physical systems.
We now characterize the optimal predictor associated with this loss.

\begin{proposition}
\label{prop:bayes}
Among all predictors measurable with respect to $\mathcal F_t$, the unique
minimizer of the population risk $\mathcal R(f)$ is given by
\[
f^\star(\mathbf X_t^{(L)})
=
\mathbb E\!\left[\mathbf Y_t^{(H)} \mid \mathcal F_t\right]
=
\bigl(
\mathbb E[X_{t+1}\mid\mathcal F_t],\dots,
\mathbb E[X_{t+H}\mid\mathcal F_t]
\bigr).
\]
\end{proposition}

\begin{proof}
Let $Y:=\mathbf Y_t^{(H)}$ which belongs to the Hilbert space
$L^2(\Omega;\mathbb R^H)$ equipped with the inner product $\langle A,B\rangle := \mathbb E[A^\top B]$, for any $A,B\in L^2(\Omega;\mathbb R^H).$
Let $\mathcal S$ be the closed linear subspace of $L^2(\Omega;\mathbb R^H)$
consisting of all $\mathcal F_t$-measurable $\mathbb R^H$-valued random variables.
For any predictor function $f$ in $\mathcal{H}$ of Y, the random vector
$Z:=f(\mathbf X_t^{(L)})$ belongs to $\mathcal S$, and
\[
\mathcal R(f)=\mathbb E\|Y-Z\|_2^2.
\]

By the projection theorem in Hilbert spaces, there exists a unique element
$\Pi_{\mathcal S}Y\in\mathcal S$ minimizing $\mathbb E\|Y-Z\|_2^2$ over $Z\in\mathcal S$.
Moreover, $\Pi_{\mathcal S}Y$ is characterized by the orthogonality condition
\[
\mathbb E\big[(Y-\Pi_{\mathcal S}Y)^\top W\big]=0
\quad\text{for all } W\in\mathcal S.
\]
By definition, conditional expectation $\mathbb E[Y\mid\mathcal F_t]$, belongs to $\mathcal S$
moreover it satisfies this orthogonality condition, since for any $\mathcal F_t$-measurable
$W$,
\begin{align*}
\mathbb E\!\left[(Y-\mathbb E[Y\mid\mathcal F_t])^\top W\right]
 & =
\mathbb E\!\left[
\mathbb E\!\left[(Y-\mathbb E[Y\mid\mathcal F_t])^\top W \,\middle|\, \mathcal F_t\right]
\right] \\
& =
\mathbb E\!\left[
W^\top\,\mathbb E\!\left[Y-\mathbb E[Y\mid\mathcal F_t]\mid \mathcal F_t\right]
\right]
=0.
\end{align*}
Therefore, $\Pi_{\mathcal S}Y=\mathbb E[Y\mid\mathcal F_t]$, which proves the claim.
\end{proof}

Proposition ~\ref{prop:bayes} shows that, under a quadratic trajectory loss, learning reduces to estimating the conditional mean of the future trajectory given the observed past.
Whether this conditional mean is informative or degenerate depends entirely on the structure of the underlying process.
This observation provides a unifying perspective on the empirical success of Transformer-based models in structured forecasting tasks and their apparent failure in other domains, which we analyze below.


A common feature of such time series is the presence of a structured and predictable component in the dynamics.
At a high level, these processes can be represented as
\[
X_{t+1} = g(\mathcal{F}_t) + \varepsilon_{t+1},
\]
where $g(\mathcal{F}_t)$ captures deterministic or slowly varying structure (e.g., seasonality, trends, physical constraints), and $\varepsilon_{t+1}$ denotes a noise term satisfying
\[
\mathbb{E}[\varepsilon_{t+1} \mid \mathcal{F}_t] = 0.
\]
\noindent In this setting, the conditional mean of the process is non-trivial and given by
\[
\mathbb{E}[X_{t+1} \mid \mathcal{F}_t] = g(\mathcal{F}_t).
\]
Consequently, the Bayes-optimal predictor characterized in Proposition \ref{prop:bayes} recovers the structured component of the dynamics.
For multi-step forecasting, the conditional mean trajectory inherits this structure and provides an informative target for learning.

This observation explains why minimizing a squared trajectory loss is appropriate for such benchmarks.
The loss directly encourages the model to approximate $g(\mathcal{F}_t)$ across future horizons, and increased model capacity improves the approximation of this conditional mean.
Architectural choices specific to Transformer-based models such as attention mechanisms, patch-based tokenization, or frequency domain decompositions serve to facilitate this approximation by improving representation and optimization, rather than altering the learning objective itself.

Crucially, in these settings, the success of Transformer based forecasting models does not rely on predicting noise, but on exploiting predictable structure present in the conditional mean of the process.
The effectiveness of trajectory forecasting under squared loss therefore hinges on the existence of such structure. \\
In the absence of a non-trivial conditional mean, as in finance time series, the same learning objective leads to qualitatively different outcomes, as we show in the next section.

\subsection*{Financial time series and forecast collapse}

We now specialize the general framework of the preceding paragraph to financial time series.

We emphasize that the learning setup remains unchanged: the model observes a finite window of past values and is trained to predict a future trajectory under a squared loss.

A standard modeling assumption in finance is that asset returns exhibit little
or no stable predictable bias given past information.
At the level of prices, this principle is commonly formalized through a
martingale-type condition, according to which the conditional expectation of
future prices given the available information coincides with the current price.
This idea dates back to early equilibrium arguments, most notably the seminal
result of \cite{Samuelson65}, who showed that properly anticipated
prices must fluctuate randomly, and was later popularized in the efficient
market hypothesis formulated by  \cite{Fama70}.

Since then, a vast empirical literature has investigated deviations from
random walk or martingale behavior in financial time series.
While formal statistical tests often reject exact random-walk hypotheses
\cite{LoMacKinlay88}, subsequent analyses have consistently shown that any
predictability in the conditional mean of returns is weak, unstable, and largely
dominated by noise, especially out of sample.
Comprehensive econometric studies confirm that most of the statistically
significant structure in asset returns is concentrated in higher-order
properties such as volatility clustering, heavy tails, or cross-sectional
effects rather than in the conditional expectation itself
\cite{Campbell97,Cont01}.

More recent work has reinforced this view in both classical econometrics and
modern machine learning settings.
Large-scale empirical studies show that return-predictive models, including
highly flexible specifications, rarely outperform simple historical averages or
martingale benchmarks in out-of-sample forecasting
\cite{GoyalWelch08}.
Even when machine learning methods are employed, their gains are primarily
observed in cross-sectional prediction rather than in the time-series
forecasting of aggregate returns
\cite{GuKellyXiu20}.
These findings suggest that, for practical forecasting purposes under squared
loss, assuming a nearly flat conditional mean for prices or a near-zero
conditional mean for returns provides a reasonable and robust baseline.

Importantly, the conclusions of the present work do not rely on an exact
martingale assumption.
They extend to a broad class of stochastic processes whose conditional mean of
the future remains close to the last observed value, possibly up to a slowly
varying drift or a weak, regime-dependent component.
In such settings, any predictable structure in the conditional mean is dominated
by stochastic innovations, so that empirical risk minimization with highly
expressive models is expected to converge toward trivial forecasts.
Additional model complexity or stochastic training effects can only introduce
variance around this baseline, without generating genuine predictive content.

While these empirical findings strongly suggest that the conditional mean of financial time series is weakly informative, they do not provide a precise understanding of the behavior of learning algorithms in this regime. In particular, they do not explain how empirical risk minimization interacts with model expressivity at the process level, where both inputs and outputs are temporal trajectories.

The contribution of the present work is to make this connection explicit. We show that, under squared trajectory loss, the apparent failure of highly expressive models in financial forecasting is not a consequence of optimization difficulties or data limitations, but a direct implication of the learning objective itself. In this regime, increasing model complexity does not uncover hidden structure, but instead amplifies noise through interpolation, leading to degraded predictive performance.

Throughout this section, $X_t$ denotes either a price (or log-price) process, and $\mathcal{G}_t$ the natural filtration representing the information available up to time $t$.
For the sake of clarity and analytical tractability, we nevertheless adopt the
standard martingale-type assumption commonly used in finance. 
Specifically, we assume that prices satisfy
\[
\mathbb{E}[X_{t+1}\mid \mathcal G_t] = X_t
\quad \text{almost surely for all } t.
\]
This assumption does not impose independence of increments and allows for
time-varying, clustered volatility and other forms of conditional heteroskedasticity.
It should therefore be understood as an idealized baseline capturing the absence
of systematic predictability in the conditional mean, rather than as a literal
description of market dynamics. A consequence of this hypothesis  is the following Corollary.

\begin{corollary}
\label{cor:bayes_flat}
Assume that $(X_t)$ is a martingale with respect to $(\mathcal G_t)$ and that
$\mathcal F_t \subseteq \mathcal G_t$.
Recall that $f^\star$ denote the Bayes-optimal predictor under squared trajectory loss. Then, for any forecast horizon $H\ge 1$,
\[
f^\star(\mathbf X_t^{(L)})
=
(X_t,\dots,X_t)\in\mathbb R^H.
\]
\end{corollary}

\begin{proof}
By Proposition~\ref{prop:bayes},
\[
f^\star(\mathbf X_t^{(L)})=\mathbb E[\mathbf Y_t^{(H)}\mid \mathcal F_t],
\qquad
\mathbf Y_t^{(H)}=(X_{t+1},\dots,X_{t+H}).
\]
By the martingale property, $\mathbb E[X_{t+h}\mid \mathcal F_t]=X_t$, this holds for all $h=1,\dots,H$, proving the claim.
\end{proof}

Notice also that, under the martingale assumption on $(X_t)$, the associated one-step
returns $R_{t+1}:=X_{t+1}-X_t$ automatically have zero conditional mean.
As a consequence, the Bayes-optimal predictor under squared loss for the future
return trajectory is trivial: for any lookback window $\mathbf R_t^{(L)}$, one has
$f^\star(\mathbf R_t^{(L)})=(0,\dots,0)\in\mathbb R^H$.

\medskip


We now pursue the analysis by comparing the prediction errors of a simple
well-specified linear model with those of a highly expressive, strongly
interpolating predictor.
While both models are trained via empirical risk minimization, their behavior
differs fundamentally when the Bayes-optimal forecast is trivial, allowing us to
isolate the contribution of model complexity to excess prediction error, beyond
the irreducible noise level.

\section{Prediction error for signals without identifiable structure}

In this section we keep the trajectory forecasting setup introduced in the previous section and we put ourselves in the hypothesis of Corollary \ref{cor:bayes_flat}.
In addition to the set  $\mathcal{H}$ we introduce a more simple set of estimator defined as follows :
consider the one-parameter linear hypothesis class
\[
\mathcal H^{\ell} := \{ f_a(u)=a\cdot x(u) \cdot (1,\cdots,1) : a\in\mathbb R, u \in \mathbb R^H\},
\]
with $x(u)$ the last coordinate of $u$ and let us denote $\hat f^{\ell}_n$ be the empirical risk minimizer over $\mathcal H^{\ell}$
trained on the sample $\mathcal D_n := \bigl\{ (\mathbf X_{t_i}^{(L)}, \mathbf Y_{t_i}^{(H)}) \bigr\}_{i=1}^n$ with squared loss, it is given for any $u \in \mathbb R^L$ by 
\begin{align}
    \hat f^{\ell}_n(u):=\hat a\cdot x(u)  \cdot (1,\cdots,1) \in \mathbb R^H , \label{fl}
\end{align} 
where $\hat a$ denotes the least-squares estimator of the scalar parameter $a$ computed over $\mathcal D_n$ (see the Appendix A.2).  Contrarily to $\mathcal H^{\ell} $, $\mathcal H $ is considered richer and contains, in particular,
nearest-neighbor–type interpolating predictors of the form: for any $u \in \mathbb R^L$
\[
\hat f_n(u) := Y_{t_{i^{\star}(u)}}^{(H)},
\qquad
i^{\star}(u)\in argmin_{1\le i\le n}\|u-\mathbf X_{t_i}^{(L)}\|_2^2,
\]
where a unique index is selected whenever the minimum is not unique.
Any sufficiently expressive model class (including large neural networks)
can approximate such functions arbitrarily well.

\noindent The two classes above correspond to fundamentally different statistical regimes.
The class $\mathcal H^{\ell}$ is parametric, with fixed dimension and vanishing estimation
variance as the sample size increases.
By contrast, $\mathcal H$ admits interpolating predictors whose effective complexity
grows with the data, in the sense that the learned function adapts to the training
sample size, and which can exactly reproduce the training outputs.

When the conditional structure of the data is simple as in the present setting,
where the Bayes predictor is linear and the future is dominated by noise the additional
expressive power of $\mathcal H$ does not reduce bias, but instead induces extra
variance by re-injecting training noise into the prediction.
This phenomenon is not specific to nearest neighbors or to any particular architecture,
but is a generic feature of highly expressive, weakly regularized models.

The following proposition formalizes this intuition by comparing empirical risk
minimization over $\mathcal H^{\ell}$ and over a strictly richer class $\mathcal H$ containing
interpolating predictors, and shows that the latter can yield a strictly larger
expected prediction error.

\begin{proposition}
\label{prop:final_compare}
Fix $L\ge 1$ and $H\ge 1$. Assume the forecasting target satisfies
\[
\mathbf{Y}_t^{(H)} = (X_t,\dots,X_t) + \varepsilon_t \in \mathbb R^H,
\qquad
\mathbb E[\varepsilon_t\mid \mathcal F_t]=0,
\qquad
\mathbb E\|\varepsilon_t\|_2^2 = H\sigma^2<\infty.
\]
Under the regularity assumptions stated in Appendix~A.2, we have
\[
\mathbb E\|\mathbf{Y}_t^{(H)}-\hat f_n(\mathbf{X}_t^{(L)})\|_2^2 \;\ge\; 2H\sigma^2,
\qquad
\mathbb E\|\mathbf{Y}_t^{(H)}-\hat f^{\ell}_n(\mathbf{X}_t^{(L)})\|_2^2
= H\sigma^2 + O\!\left(\frac{H\sigma^2}{n}\right).
\]
Consequently, for $n$ large enough,
\[
\mathbb E\|\mathbf{Y}_t^{(H)}-\hat f_n(\mathbf{X}_t^{(L)})\|_2^2
>
\mathbb E\|\mathbf{Y}_t^{(H)}-\hat f^{\ell}_n(\mathbf{X}_t^{(L)})\|_2^2.
\]
\end{proposition}

\begin{proof}
Under the assumptions of the proposition, each training pair indexed by $i$
in the sample $\mathcal D_n$ satisfies
\[
\mathbf{Y}_{t_i}^{(H)} = (x(\mathbf{X}_{t_i}^{(L)}),\cdots,x(\mathbf{X}_{t_i}^{(L)})) + \varepsilon_{t_i},
\quad
\mathbb E[\varepsilon_{t_i}\mid \mathbf{X}_{t_i}^{(L)} ]=0_H,
\quad
\mathbb E\|\varepsilon_{t_i}\|_2^2 = H\sigma^2.
\]
Now, let, $(U \in \mathbb R^L, V \in \mathbb R^H)$ denotes an independent test pair, distributed identically to the training data and independent of $\mathcal D_n$. In particular, $V$ satisfies $V=(x(U),\cdots,x(U))+\epsilon$ with $\mathbb E(V|U)=(x(U),\cdots,x(U))\in \mathbb R^H$. \\ 
First we focus on the loss generated by the rich,  interpolated predictor : \\
\noindent\textit{Interpolating predictor.}
Define the nearest-neighbor index, of $U$ in the set of data $(\mathbf{X}_{t_i}^{(L)}, 1 \leq i \leq n)$ 
\[
i^{\star}(U)\in \operatorname*{arg\,min}_{1\le i\le n}\|U-\mathbf{X}_{t_i}^{(L)}\|_2
\]
(with a fixed deterministic choice when several minimizers occur) and the interpolating predictor
\[
\hat f_n(U) :=  \mathbf{Y}_{t_{i^{\star}(U)}}^{(H)}.
\]
Then 
\[
V-\hat f_n(U)= V-\mathbf{Y}_{t_{i^{\star}(U)}}^{(H)}
=
(x(U)-x(\mathbf{X}_{t_i^*(U)}^{(L)}))(1,\dots,1) + (\varepsilon-\varepsilon_{t_{i^{\star}(U)}}).
\]
Expanding the squared norm and taking expectations, the cross term vanishes because of independence and $\mathbb E[\varepsilon-\varepsilon_{t_{i^{\star}}}\mid U,\mathbf{X}_{t_{i^{\star}(U)}}]=0$; hence
\[
\mathbb E\|V-\hat f_n(U)\|_2^2
=
H\,\mathbb E\big[(x(U)-x(\mathbf{X}_{t_{i^{\star}(U)}}^{(L)}))^2\big]
+
\mathbb E\|\varepsilon-\varepsilon_{t_{i^{\star}(U)}}\|_2^2
\;\ge\;
\mathbb E\|\varepsilon-\varepsilon_{t_{i^{\star}(U)}}\|_2^2.
\]
Moreover, $\varepsilon$ (test noise) is independent of $\varepsilon_{t_{i^{\star}}}$ (a training noise term),
both are centered, and
$\mathbb E\|\varepsilon\|_2^2=\mathbb E\|\varepsilon_{t_{i^{\star}(U)}}\|_2^2=H\sigma^2$,
so
\[
\mathbb E\|\varepsilon-\varepsilon_{t_{i^{\star}(U)}}\|_2^2
=
\mathbb E\|\varepsilon\|_2^2+\mathbb E\|\varepsilon_{t_{i^{\star}(U)}}\|_2^2
=
2H\sigma^2.
\]
Therefore $\mathbb E\|V-\hat f_n(U)\|_2^2 \ge 2H\sigma^2$. \\

\medskip
\noindent We now repeat the same analysis, replacing the interpolating predictor with a
simple parametric regressor. \\

\medskip
\noindent\textit{Simple parametric regressor.}
Consider the one-parameter class, with predictor $\hat f_n^{\ell}$ given by \eqref{fl}, then
\[
V-\hat f_n^{\ell}(U)=(1-\hat a)(x(U),\dots,x(U))+\varepsilon,
\]
so taking expectations and using $\mathbb E[\varepsilon\mid U]=0$ yields
\[
\mathbb E\|V-\hat f_n^{\ell}(U)\|_2^2
=
H\sigma^2 + H\,\mathbb E\big[(\hat a-1)^2 x(U)^2\big].
\]
Under the regularity assumptions (A1)--(A5) stated in Appendix~A.2,
self-normalized martingale arguments yield
\[
\mathbb E\big[(\hat a-1)^2\big]
=
O\!\left(\frac{\sigma^2}{n}\right).
\]
More precisely, Appendix~A.2 proves this bound using a martingale
concentration argument based on Freedman-type inequalities together with
a mild non-degeneracy condition ensuring that the normalization
$\sum_{i=1}^n x(\mathbf X_{t_i}^{(L)})^2$ grows linearly in $n$.
Consequently,
\[
\mathbb E\|Y-\hat f_n^{\ell}(U)\|_2^2
=
H\sigma^2 + O\!\left(\frac{H\sigma^2}{n}\right).
\]
This completes the proof of Proposition~\ref{prop:final_compare}.
\end{proof}

\section{Numerical simulation: foreign exchange forecasting}

We illustrate our theoretical results on high-frequency foreign exchange data, focusing on EUR/USD intraday trajectories sampled every 30 seconds. The dataset spans the period from December 31, 2020 (13:00) to July 31, 2025 (16:59:30). \\

The dataset consists of approximately $1{,}160$ trading days of EUR/USD exchange rates, restricted to the intraday window from 13:00 to 18:00 (UTC). Each day is represented as a univariate time series of prices, from which we extract rolling forecasting windows. For each day $t$, we construct:
\begin{itemize}
    \item an input sequence $\mathbf{X}_t^{(L)} \in \mathbb{R}^{L \times d}$, corresponding to the historical trajectory (from 13:00 to 16:45),
    \item a target sequence $\mathbf{Y}_t^{(H)} \in \mathbb{R}^{H}$, corresponding to the future trajectory (from 16:45 to 18:00).
\end{itemize}


\begin{figure}[t]
\centering
\includegraphics[width=\textwidth]{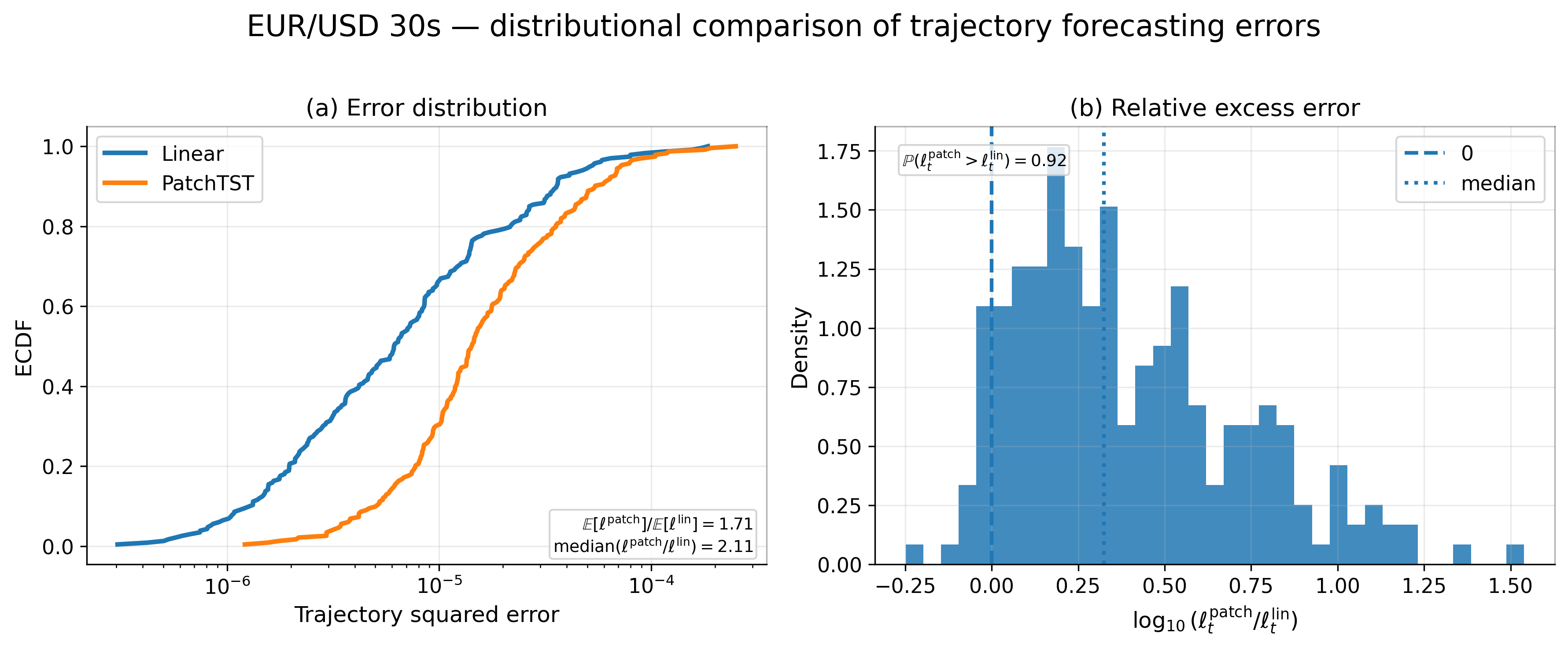}
\caption{
Distribution of trajectory forecasting errors on EUR/USD 30-second data.
Panel (a): empirical cumulative distribution functions.
Panel (b): distribution of the log-ratio of errors.
}
\label{fig:eurusd_ecdf}
\end{figure}

In our experiments, we fix $L = 451$ (past observations) and $H = 30$ (forecast horizon). The dataset contains $1159$ usable forecasting windows, each corresponding to one intraday EUR/USD session. We split these windows chronologically into training (70\%), validation (10\%), and test (20\%) subsets. In the baseline configuration, this yields $n = 811$ training windows, so that the predictor $\hat f_n$ is learned from $811$ examples.

Let $\mathcal D_{\mathrm{test}}$ denote the set of test windows. Each element $t \in \mathcal D_{\mathrm{test}}$ corresponds to one forecasting instance, consisting of a past trajectory $\mathbf X_t^{(L)}$ and its associated future trajectory $\mathbf Y_t^{(H)}$. The empirical quantities reported below are all computed over this test collection of windows.

\noindent The aim is to compare the two predictors $\hat f_n$ and $\hat f_n^{\ell}$ by analyzing the distribution of their trajectory-level errors over the test sample.

\noindent For the more expressive predictor $\hat f_n$, we use a PatchTST model, which follows a standard configuration for long-horizon forecasting, with patch-based tokenization of the input sequence, multi-head self-attention, and shared weights across time (see \cite{nie2023patchtst}).

Before commenting on Figure~\ref{fig:eurusd_ecdf}, we introduce the notation used throughout the analysis.

For each $t \in \mathcal D_{\mathrm{test}}$ we consider the trajectory-level squared loss
\[
\ell_t(f)
:=
\left\|\mathbf{Y}_t^{(H)} - f(\mathbf{X}_t^{(L)})\right\|_2^2.
\]
In particular, we define the losses associated with the two predictors under consideration:
\begin{align}
\ell_t^{\mathrm{lin}}
&:=
\ell_t(\hat f_n^{\ell})
=
\left\|\mathbf{Y}_t^{(H)} - \hat f_n^{\ell}(\mathbf{X}_t^{(L)})\right\|_2^2, \\ \ell_t^{\mathrm{patch}}
&:=
\ell_t(\hat f_n)
=
\left\|\mathbf{Y}_t^{(H)} - \hat f_n(\mathbf{X}_t^{(L)})\right\|_2^2.
\end{align}
 Let $N = |\mathcal D_{\mathrm{test}}|$ denote the number of test windows (here $N \approx 233$), we approximate the risk using the empirical test risk
\[
\hat{\mathcal R}_{\mathrm{test}}(f)
=
\frac{1}{N}
\sum_{t \in \mathcal D_{\mathrm{test}}}
\ell_t(f),
\]
which corresponds to the average trajectory error over the test sample.
In particular, the ratio displayed in Figure~\ref{fig:eurusd_ecdf}(a) corresponds to
\[
\frac{\hat{\mathcal R}_{\mathrm{test}}(\hat f_n)}
{\hat{\mathcal R}_{\mathrm{test}}(\hat f_n^{\ell})}
=
\frac{\frac{1}{N}\sum_t \ell_t^{\mathrm{patch}}}
{\frac{1}{N}\sum_t \ell_t^{\mathrm{lin}}}.
\]

\noindent We also report the empirical frequency in Figure~\ref{fig:eurusd_ecdf}(b), as 
\[
\mathbb{P}_{\mathrm{test}}\!\left(
\ell_t^{\mathrm{patch}} > \ell_t^{\mathrm{lin}}
\right)
=
\frac{1}{N}
\sum_{t \in \mathcal D_{\mathrm{test}}}
\mathbf{1}\!\left\{
\ell_t^{\mathrm{patch}} > \ell_t^{\mathrm{lin}}
\right\},
\]
which measures how often the Transformer-based predictor incurs a larger trajectory error than the linear predictor.

\medskip
\noindent Figure~\ref{fig:eurusd_ecdf} reveals a clear and systematic degradation of the Transformer-based predictor compared to its linear counterpart.

In panel~(a), the empirical cumulative distribution function of $\ell_t^{\mathrm{patch}}$ is uniformly shifted to the right of that of $\ell_t^{\mathrm{lin}}$. This indicates that, for essentially all quantiles, the trajectory-level prediction error of $\hat f_n$ exceeds that of $\hat f_n^{\ell}$. In other words, the degradation is not driven by a few extreme events but is pervasive across the entire distribution. This observation is confirmed by the ratio of empirical risks, with
\[
\frac{\hat{\mathcal R}_{\mathrm{test}}(\hat f_n)}
     {\hat{\mathcal R}_{\mathrm{test}}(\hat f_n^{\ell})}
\approx 1.71,
\]
showing that the average trajectory error of the PatchTST predictor is nearly three times larger than that of the linear model.

Panel~(b) provides a more granular view by considering the distribution of the log-ratio
\[
\log_{10}\!\left(\frac{\ell_t^{\mathrm{patch}}}{\ell_t^{\mathrm{lin}}}\right).
\]
The mass of the distribution lies predominantly above zero, and we observe that
\[
\mathbb{P}_{\mathrm{test}}\!\left(\ell_t^{\mathrm{patch}} > \ell_t^{\mathrm{lin}}\right) \approx 0.92,
\]
meaning that the Transformer predictor yields larger errors on approximately $94\%$ of the forecasting windows.

Together, these results show that the increased expressivity of $\hat f_n$ does not translate into improved predictive performance in this setting. Instead, it induces additional variability, resulting in systematically larger trajectory-level errors.

Importantly, the Transformer-based predictor is not excessively overparameterized in our experiments. We deliberately use a moderately sized PatchTST architecture. In preliminary experiments, increasing model capacity further widened the performance gap with the simple linear benchmark. The reported configuration should therefore be viewed as a conservative choice, which already suffices to exhibit the phenomenon predicted by the theory.


\section{Conclusion}

In this work, we provided a process-level explanation for a widely observed empirical phenomenon: highly expressive models trained by empirical risk minimization under squared loss tend to produce trivial forecasts when applied to financial time series.

By formulating forecasting as a trajectory-to-trajectory learning problem, we showed that empirical risk minimization targets the conditional mean trajectory of the future given the past. While this objective is well suited to structured time series such as energy consumption or traffic data, it becomes effectively degenerate for financial processes whose innovations are not predictably biased.

Our analysis clarifies that this behavior is neither a limitation of a specific architecture, such as Transformers, nor a failure of optimization. On the contrary, increased model capacity and larger datasets accelerate convergence toward the Bayes-optimal predictor, which is flat for prices and zero for returns under standard financial assumptions.

These theoretical insights are supported by empirical results on high-frequency EUR/USD data. We show that a Transformer-based predictor (PatchTST) yields systematically larger trajectory-level errors than a simple linear benchmark, with the discrepancy being pervasive across forecasting windows. Importantly, this effect persists even for moderately sized architectures, while its magnitude increases with model capacity, which is consistent with a variance-driven explanation of the excess error.

Beyond this negative result, our analysis suggests several directions for future work. First, models that aim to learn the full conditional distribution of future trajectories, rather than their conditional mean, may remain meaningful in financial settings. Diffusion-based models and other probabilistic forecasting approaches fall naturally into this category, as they are designed to capture uncertainty and higher-order moments rather than point forecasts. Second, our framework provides a natural way to reason about the signal-to-noise ratio required for successful trajectory forecasting. For structured benchmarks such as electricity consumption, one may ask how much stochastic noise can be injected before the conditional mean becomes effectively trivial and expressive models lose predictive power. Characterizing this transition could help bridge the gap between synthetic benchmarks and real-world financial data.

More broadly, our results suggest that meaningful progress in financial machine learning is unlikely to come from further architectural refinements alone. Instead, it requires rethinking the learning objective, the prediction target, or the decision problem itself. Making these limitations explicit is, in our view, a necessary step toward more principled and effective modeling approaches for financial time series.


\section*{Appendix}

\subsection*{A.1 Consistency of ERM within the model class}
\begin{proof}
Let
\[
\Delta_n := \sup_{f\in\mathcal H}\bigl|\widehat{\mathcal R}_n(f)-\mathcal R(f)\bigr|.
\]
By definition of $\Delta_n$, for any $f\in\mathcal H$,
\[
\mathcal R(f) \le \widehat{\mathcal R}_n(f)+\Delta_n,
\qquad
\widehat{\mathcal R}_n(f) \le \mathcal R(f)+\Delta_n.
\]

\noindent By ERM optimality,
\[
\widehat{\mathcal R}_n(\hat f_n) \le \widehat{\mathcal R}_n(f)
\quad\text{for all } f\in\mathcal H.
\]
Hence, for any $f\in\mathcal H$,
\[
\mathcal R(\hat f_n)
\le
\widehat{\mathcal R}_n(\hat f_n)+\Delta_n
\le
\widehat{\mathcal R}_n(f)+\Delta_n
\le
\mathcal R(f)+2\Delta_n.
\]
Subtracting $\mathcal R(f)$ from both sides and taking the infimum over
$f\in\mathcal H$ yields
\[
0 \le \mathcal R(\hat f_n)-\inf_{f\in\mathcal H}\mathcal R(f) \le 2\Delta_n.
\]
Since $\mathcal H$ is Glivenko--Cantelli, we have $\Delta_n \rightarrow 0$ in probability, and therefore
\[
\left|\mathcal R(\hat f_n)-\inf_{f\in\mathcal H}\mathcal R(f)\right| \rightarrow 0, \quad \text{in\ probability as } n \to \infty.
\]

\end{proof}

\subsection*{A.2 Control of the least-squares coefficient: a self-normalized martingale argument}
\label{app:A2}

In this appendix, we justify the bound
\[
\mathbb E[(\hat a-1)^2]=O\!\left(\frac{\sigma^2}{n}\right)
\]
used in Proposition~\ref{prop:final_compare}, without assuming independent or identically distributed data. Throughout this appendix, to simplify notation, we write
\[
x_i := x(\mathbf X_{t_i}^{(L)}),
\qquad
\varepsilon_i := \varepsilon_{t_i},
\qquad
Y_i := x_i \mathbf 1_H + \varepsilon_i,
\]
and we focus on the scalar least-squares fit along the direction $\mathbf 1_H$.
(The factor $H$ only produces multiplicative constants and is handled in the main text.)

\subsubsection*{Least-squares estimator}

Consider the one-parameter model
\[
f_a(u) = a\,x(u)\,\mathbf 1_H,
\qquad a\in\mathbb R.
\]
The empirical risk minimizer $\hat a$ is given by the usual least-squares formula
\[
\hat a
=
\frac{\sum_{i=1}^n x_i \langle Y_i,\mathbf 1_H\rangle}
{\sum_{i=1}^n x_i^2 \|\mathbf 1_H\|_2^2}
=
1 + \frac{\sum_{i=1}^n x_i \xi_i}{\sum_{i=1}^n x_i^2}.
\]
where we introduced the scalar innovation
\[
\xi_i
:=
\frac{\langle \varepsilon_i,\mathbf 1_H\rangle}{\|\mathbf 1_H\|_2^2}
=
\frac{1}{H}\sum_{h=1}^H \varepsilon_{i,h}.
\]
Hence
\begin{equation}
\label{eq:a_hat_ratio_app}
\hat a - 1 = \frac{M_n}{V_n},
\qquad
M_n := \sum_{i=1}^n x_i \xi_i,
\qquad
V_n := \sum_{i=1}^n x_i^2.
\end{equation}

\subsubsection*{Assumptions}

Let $(\mathcal F_i)_{i\ge0}$ be a filtration such that $x_i$ is $\mathcal F_{i-1}$-measurable and
$\xi_i$ is $\mathcal F_i$-measurable.
We assume:

\begin{enumerate}
\item[(A1)] {Martingale-difference condition:}
\[
\mathbb E[\xi_i\mid\mathcal F_{i-1}] = 0
\quad \text{a.s. for all } i.
\]

\item[(A2)] {Uniform conditional variance bound:}
\[
\mathbb E[\xi_i^2\mid\mathcal F_{i-1}] \le \sigma^2
\quad \text{a.s. for all } i.
\]

\item[(A3)] {Conditional symmetry:}
\[
\mathcal L(\xi_i\mid\mathcal F_{i-1})
=
\mathcal L(-\xi_i\mid\mathcal F_{i-1})
\quad \text{a.s. for all } i.
\]

\item[(A4)] {Small-ball (non-degeneracy) condition:}
there exist constants $b>0$ and $p\in(0,1)$ such that
\[
\mathbb P(|x_i|\ge b \mid \mathcal F_{i-1})\ge p
\quad \text{a.s. for all } i.
\]

\item[(A5)] {Lower-tail integrability of the self-normalizer:}
there exists a constant $c>0$ such that
\[
\sup_{n\ge1}
\mathbb E\!\left[\frac{n}{V_n}\mathbf 1_{\{V_n<cn\}}\right]
<\infty.
\]
\end{enumerate}

Assumptions (A1) to (A3) encode the absence of predictable structure in the innovations,
while (A4) prevents the design from collapsing by ensuring that the regressors
remain nondegenerate with positive probability.
Assumption (A5) controls the lower tail of the normalization
$V_n=\sum_{i=1}^n x_i^2$ and guarantees the integrability of its inverse.
In other words, even on the rare events where $V_n$ falls below its typical
linear scale (see \eqref{eq:Vn_tail}), the factor $1/V_n$ remains integrable. In financial terms, for example, (A5) excludes pathological situations where the
observed price level remains arbitrarily close to zero over a long time
interval.


Under (A1), the process $(M_k)_{k\le n}$ defined in \eqref{eq:a_hat_ratio_app} is a martingale with respect to $(\mathcal F_k)$, with increments
\[
\Delta M_i = x_i \xi_i.
\]
Its predictable quadratic variation is
\[
\langle M\rangle_n
=
\sum_{i=1}^n \mathbb E[(\Delta M_i)^2\mid\mathcal F_{i-1}]
=
\sum_{i=1}^n x_i^2\,\mathbb E[\xi_i^2\mid\mathcal F_{i-1}].
\]
As by (A2),
\begin{equation}
\label{eq:bracket_dom_app}
\langle M\rangle_n \le \sigma^2 V_n
\qquad \text{a.s}
\end{equation}

\noindent then 
\[
\mathbb P\!\left(\frac{|M_n|}{V_n}\ge t\right)
=
\mathbb P\!\left(|M_n|\ge tV_n,\ \langle M\rangle_n \le \sigma^2 V_n\right).
\]
Also from the Bercu-Touati \cite{BercuTouati2008} inequality: for all $u>0$ and $v>0$,
\[
\mathbb P\!\left(|M_n|\ge u,\ \langle M\rangle_n \le v\right)
\le
2\exp\!\left(-\frac{u^2}{2v}\right).
\]
Applying this inequality with $u=tV_n$ and $v=\sigma^2 V_n$
yields
\[
\mathbb P\!\left(\frac{|M_n|}{V_n}\ge t \,\middle|\, V_n\right)
\le
2\exp\!\left(-\frac{t^2 V_n}{2\sigma^2}\right).
\]
Taking expectation with respect to $V_n$, we obtain the following right tale for $\frac{|M_n|}{V_n}$ 
\begin{equation}
\label{eq:ratio_tail_app_correct}
\mathbb P\!\left(\frac{|M_n|}{V_n}\ge t\right)
\le
2\,\mathbb E\!\left[
\exp\!\left(-\frac{t^2 V_n}{2\sigma^2}\right)
\right].
\end{equation}
Using the identity $\mathbb E[Z^2]=\int_0^\infty \mathbb P(Z^2\ge s)\,ds$ for $Z\ge0$ and
\eqref{eq:a_hat_ratio_app},
\[
\mathbb E[(\hat a-1)^2]
=
\int_0^\infty
\mathbb P\!\left(\frac{|M_n|}{V_n}\ge \sqrt{s}\right)\,ds.
\]
Applying \eqref{eq:ratio_tail_app_correct} and Fubini's theorem yields
\begin{equation}
\label{eq:moment_Vn}
\mathbb E[(\hat a-1)^2]
\le
4\sigma^2\,\mathbb E\!\left[\frac{1}{V_n}\right].
\end{equation}

To obtain an upper bound for $\mathbb E\!\left[\frac{1}{V_n}\right]$, we first control
the lower tail of the normalization
\[
V_n = \sum_{i=1}^n x_i^2 .
\]

Introduce the indicator variables
\[
I_i := \mathbf 1_{\{|x_i|\ge b\}} .
\]
Under the small-ball condition (A4), we have almost surely
\[
V_n = \sum_{i=1}^n x_i^2
\;\ge\;
b^2 \sum_{i=1}^n I_i .
\]

Define
\[
Z_i := I_i - \mathbb E[I_i\mid\mathcal F_{i-1}],
\qquad
S_n := \sum_{i=1}^n Z_i .
\]
The process $(S_n)$ is a martingale with bounded increments
$|Z_i|\le 1$.

Since $\mathbb E[I_i\mid\mathcal F_{i-1}] \ge p$ almost surely by (A4),
we can write
\[
\sum_{i=1}^n I_i
=
\sum_{i=1}^n \mathbb E[I_i\mid\mathcal F_{i-1}]
+
S_n
\;\ge\;
pn + S_n .
\]

Therefore, for any $\eta\in(0,p)$,
\[
\mathbb P\!\left(\sum_{i=1}^n I_i \le (p-\eta)n\right)
\le
\mathbb P(S_n \le -\eta n).
\]

Applying the Azuma--Hoeffding inequality to the martingale $(S_n)$
yields
\[
\mathbb P(S_n \le -\eta n)
\le
\exp(-2\eta^2 n).
\]

Since $V_n \ge b^2 \sum_{i=1}^n I_i$, we obtain the lower-tail bound
\begin{equation}
\label{eq:Vn_tail}
\mathbb P(V_n \le b^2(p-\eta)n)
\le
\exp(-2\eta^2 n).
\end{equation}

Let $c := b^2(p-\eta) > 0$.
We now decompose
\[
\mathbb E\!\left[\frac1{V_n}\right]
=
\mathbb E\!\left[\frac1{V_n}\mathbf 1_{\{V_n\ge cn\}}\right]
+
\mathbb E\!\left[\frac1{V_n}\mathbf 1_{\{V_n<cn\}}\right].
\]

{Contribution of the typical event :}
On the event $\{V_n\ge cn\}$ we have $1/V_n \le 1/(cn)$, hence
\[
\mathbb E\!\left[\frac1{V_n}\mathbf 1_{\{V_n\ge cn\}}\right]
\le
\frac1{cn}.
\]

{Contribution of the rare event : }
For the second term we use Assumption~(A5), which controls the
integrability of the inverse normalization.
Indeed,
\[
\mathbb E\!\left[\frac1{V_n}\mathbf 1_{\{V_n<cn\}}\right]
=
\frac1n
\mathbb E\!\left[
\frac{n}{V_n}\mathbf 1_{\{V_n<cn\}}
\right].
\]

Assumption~(A5) ensures that
\[
\sup_{n\ge1}
\mathbb E\!\left[
\frac{n}{V_n}\mathbf 1_{\{V_n<cn\}}
\right]
<\infty ,
\]
so there exists a constant $C>0$ such that for all $n$
\[
\mathbb E\!\left[\frac1{V_n}\mathbf 1_{\{V_n<cn\}}\right]
\le
\frac{C}{n}.
\]

Note that the tail bound \eqref{eq:Vn_tail} shows that the event
$\{V_n<cn\}$ is exponentially unlikely, which confirms that the
contribution of this term is negligible.

\paragraph{Conclusion.}
Combining the two bounds yields
\[
\mathbb E\!\left[\frac1{V_n}\right]
\le
\frac1{cn} + \frac{C}{n}
=
O\!\left(\frac1n\right).
\]

Finally, combining this estimate with \eqref{eq:moment_Vn}, we obtain
\[
\mathbb E[(\hat a-1)^2]
\le
4\sigma^2\,\mathbb E\!\left[\frac1{V_n}\right]
=
O\!\left(\frac{\sigma^2}{n}\right),
\]
which proves the claimed bound. \hfill $\square$

\section*{Acknowledgments}

The author thanks Caroline Kalla (IUT d’Orléans, France) for discussions on forecasting in weakly structured financial time series, and for providing and preprocessing the high-frequency EUR/USD dataset used in the empirical analysis.


\bibliographystyle{plainnat}

\end{document}